\PassOptionsToPackage{notheorems}{jmlrutils}
\documentclass[11pt,oneside]{jmlr}
\usepackage{times,mathtools,array,booktabs}
\newtheorem{theorem}{Theorem}[section]
\newtheorem{lemma}[theorem]{Lemma}
\newtheorem{proposition}[theorem]{Proposition}
\newtheorem{corollary}[theorem]{Corollary}
\newtheorem{definition}[theorem]{Definition}
\newtheorem{remark}[theorem]{Remark}
\newenvironment{proof}[1][Proof]{\par\noindent\textbf{#1.}\ }{\hfill$\square$\par}
\newcommand{\R}{\mathbb R}
\newcommand{\E}{\mathbb E}
\newcommand{\Prob}{\mathbb P}
\newcommand{\LOO}{\mathcal O}
\newcommand{\conv}{\operatorname{conv}}
\newcommand{\diam}{\operatorname{diam}}
\newcommand{\Reg}{\operatorname{Reg}}
\newcommand{\ip}[2]{\langle #1,#2\rangle}
\newcommand{\norm}[1]{\lVert #1\rVert_2}
\newcommand{\calE}{\mathcal E}
\newcommand{\calH}{\mathcal H}
\newcommand{\one}{\mathbf 1}

\DeclareMathOperator*{\argmin}{arg\,min}
\numberwithin{equation}{section}
\hypersetup{pdftitle={Sharp Oracle-Regret Tradeoffs for Projection-Free Online Convex Optimization},
 pdfsubject={Dimension-free linear-oracle complexity under oracle-only feasibility},
 pdfkeywords={online convex optimization, linear optimization oracle, minimax regret, feasibility}}
\jmlrproceedings{}{}
\jmlrpages{}
\jmlryear{}
\title[Sharp Oracle--Regret Tradeoffs]{Sharp Oracle--Regret Tradeoffs for Projection-Free Online Convex Optimization}
\author[Vaneet Aggarwal]{\Name{Vaneet Aggarwal}\\\addr Purdue University\\West Lafayette, Indiana, USA}
\hypersetup{pdfauthor={Vaneet Aggarwal}}
\begin{document}
\maketitle
\begin{abstract}
We characterize the regret attainable in online convex optimization when access to the feasible set is limited to an exact linear optimization oracle. The learner is given an inscribed ball and a diameter bound and must remain feasible on every consistent instance. For convex $G$-Lipschitz losses, diameter at most $D$, a total allowance of $Q$ oracle calls, and a strict limit of $B$ calls per round, the dimension-free minimax expected regret is $\Theta(GD\max\{\sqrt T,T/(1+\min\{Q,BT\})^{1/4}\})$. The lower bound applies to arbitrary randomized learners. Universal feasibility first forces each action into the hull of the supplied ball and the preceding oracle replies. A fixed-body construction then couples fresh phase directions to a shared simplex, making useful replies costly repeatedly even though all losses have a common minimizer. A counted approximate-gradient method with interleaved blocks attains the matching rate. Total-budget and strict per-round guarantees follow as special cases, including the $T^{3/4}$ rate with one call per round and the quadratic total budget needed for $\sqrt T$ regret. For prescribed smoothness $\beta$, an analytic construction yields a curvature-dependent lower bound and identifies the threshold above which the general characterization remains sharp.
\end{abstract}

\section{Introduction}\label{sec:intro}
Online convex optimization compares a sequence of feasible decisions with the best fixed decision in hindsight. For convex $G$-Lipschitz losses on a set of diameter $D$, projected gradient descent guarantees $O(GD\sqrt T)$ regret \citep{zinkevich2003}. When projection is expensive, a linear optimization oracle offers a simpler way to obtain feasible points. Online Frank--Wolfe methods achieve $O(GDT^{3/4})$ regret with one oracle call per round \citep{hazan2012,weibel2026}, and blocked approximate-projection methods trade total oracle work against regret \citep{garber2022,lu2025decentralized}.

Oracle work has two distinct resource constraints. A total budget limits the number of calls over the entire horizon, whereas a strict per-round limit controls how much work can be completed before the next decision. A method that uses few calls on average may still require a long computation at a block boundary. The converse question is equally important: can a learner improve the tradeoff by choosing arbitrary queries, storing all previous replies, or using more computation between calls?

We study both questions in an oracle-only model. The learner knows an inscribed ball and a diameter upper bound, and obtains all further information about the feasible set from a fixed exact linear oracle. Feasibility is required on every instance consistent with this information. For total budget $Q$ and strict per-round budget $B$, define the effective total allowance $Q_*=\min\{Q,BT\}$. Our main result is
\begin{equation}\label{eq:intro-rate}
 \mathcal R_{\mathrm{joint}}(T,Q,B;G,D)
 =\Theta\!\left(GD\max\!\left\{\sqrt T,\frac{T}{(Q_*+1)^{1/4}}\right\}\right).
\end{equation}
The guarantee is uniform over finite dimensions and supplied balls. Taking $B=\infty$ recovers the total-budget problem; taking $Q=BT$ recovers the strict per-round problem. In particular, a constant number of calls per round gives order $T^{3/4}$ regret, while order $T$ calls per round and order $T^2$ calls in total suffice for order $\sqrt T$ regret.

\paragraph{Technical novelty: deriving the feasible-action restriction.}
The lower bound applies to every admissible randomized learner, rather than to a prescribed update family. Let $S$ be the finite set of oracle replies available before an action. On a finite-vertex body, the smaller body $\conv(C\cup S)$, where $C$ is the supplied ball, admits an exact oracle that reproduces those replies with the same tie-breaking. Feasibility on indistinguishable instances therefore forces the action into this hull. To establish the claim for randomized queries, we fix an alternative body and oracle for every subset of the vertex set, intersect the corresponding probability-one feasibility events, and only then select the subset generated by the transcript. This order of quantifiers turns universal feasibility into a geometric constraint while leaving the learner's query rule and intermediate computation unrestricted.

\paragraph{Technical novelty: recurring oracle cost against one comparator.}
The standard simplex support estimate penalizes a point supported on $N$ vertices by a norm of order $N^{-1/2}$ \citep{lan2013,jaggi2013}. An online lower bound must additionally account for reuse of earlier vertices. Our construction couples independent cube blocks to one shared simplex. Each phase reveals a direction in a fresh cube block. Oracle points acquired before the phase are poorly aligned with that direction with high probability. Points acquired during the phase can align with it, but placing substantial weight on only a few such points incurs the simplex penalty. With $m$ phases, making most phases inexpensive requires order $m$ new oracle points per phase. Thus $Q_*\asymp m^2$, while regret is of order $T/\sqrt m$, producing the fourth-root dependence in \eqref{eq:intro-rate}. All phases share one minimizer, so the lower bound is for static regret.

\paragraph{Achievability and smoothness.}
The upper analysis matches this information cost under both resource constraints. A comparator-uniform distance certificate bounds the length of each feasible update. Balanced blocks then expose the tradeoff between the number of decisions and the accuracy of each update. Two interleaved copies distribute the updates over physical rounds, satisfying the total allowance and the strict cap simultaneously. For smooth losses, replacing the norm penalty by an analytic radial penalty changes the support cost from order $N^{-1/2}$ to order $N^{-1}$ in its quadratic regime. This yields an explicit dependence of the lower bound on the prescribed smoothness.

\paragraph{Contributions.}
\begin{enumerate}
 \item Theorem~\ref{thm:joint} gives the matching lower and upper bounds under simultaneous total and per-round budgets. Its fixed-body lower bound covers arbitrary randomized queries and allows full loss feedback. Lemma~\ref{lem:hull} and Proposition~\ref{prop:hp} give the feasibility reduction and the recurring phase construction.
 \item Propositions~\ref{prop:finite-upper} and~\ref{prop:schedule} give finite-call implementations and their physical-round schedules. Corollaries~\ref{cor:total}--\ref{cor:polynomial} recover the total-budget tradeoff, the strict per-round rate, and the oracle requirements for a target regret. Table~\ref{tab:comparison} relates these consequences to previous upper bounds.
 \item Theorem~\ref{thm:smooth} gives two-sided bounds for the smooth class under the same joint budgets. Corollary~\ref{cor:smooth-threshold} identifies when the general minimax rate remains sharp and gives the curvature-sensitive lower rate below that threshold. Corollary~\ref{cor:smooth-target} translates it into necessary oracle budgets.
\end{enumerate}

\section{Related work}\label{sec:related}
\paragraph{Total oracle budgets.}
Blocking and approximate projection provide the basic regret--computation tradeoff. With a linear total number of oracle calls, \citet{garber2022} obtain $O(T^{3/4})$ regret, including an adaptive-regret guarantee. More generally, the tradeoff $O(T^{1-\theta/2})$ regret with $O(T^{2\theta})$ total calls, $0\le\theta\le1$, follows from \citet{lu2025decentralized}. Their framework treats upper-linearizable objectives, with convex optimization as a special case; we report that convex case here. Theorem~\ref{thm:joint} establishes a matching lower bound for this entire dimension-free tradeoff and makes the total budget an exact input rather than an asymptotic operation count.

\paragraph{Strict per-round budgets.}
Online Frank--Wolfe already achieves $O(T^{3/4})$ regret with one call in every round \citep{hazan2012,weibel2026}. The same order is obtained from the total-budget tradeoff at $\theta=1/2$, but a total or amortized count alone does not impose a roundwise cap. We characterize every integer cap $B\ge1$, obtaining $\Theta(\max\{\sqrt T,T^{3/4}B^{-1/4}\})$, and combine it with an independent total allowance $Q$. The upper construction uses overlapping block computations, as in \citet{garber2022}, and an independent-copy schedule related to delayed online learning \citep{joulani2013}. The additional result is a matching converse and an explicit schedule for arbitrary simultaneous caps.

\paragraph{Loss regularity and constraint information.}
Randomized perturbed-leader methods exploit linearity or smoothness to obtain different regret rates \citep{kalai2005,hazan2020}. The corresponding linear-oracle bounds retain dimension dependence: for linear losses, one call per round gives $O(\sqrt{dT})$ normalized regret; for smooth convex losses, the amortized one-call regime gives $O((\sqrt d+\beta D/G)T^{2/3})$. Our smooth analysis instead quantifies a lower bound uniform over dimensions, with the dependence on $\beta$ displayed. Membership and separation methods \citep{mhammedi2022,mhammedi2025} use different constraint information and are not included in the linear-oracle comparison.

\paragraph{Lower-bound mechanism and proof ingredients.}
Offline linear-oracle lower bounds use the difficulty of approximating a dense simplex point by a sparse convex combination \citep{lan2013,jaggi2013}. An optimal nonsmooth offline method uses $O(\varepsilon^{-2})$ calls \citep{thekumparampil2020}; repeating one fixed offline objective therefore does not explain the online fourth-root tradeoff. The present construction makes the support cost recur as new phase information arrives. The feasibility reduction establishes the support restriction for arbitrary learners, and a conditional Hoeffding argument handles adaptive query histories \citep{hoeffding1963}. On the upper side, the distance potential is inherited from online gradient descent \citep{zinkevich2003}, while blocking, approximate projection, and the Frank--Wolfe gap certificate follow \citet{garber2022} and \citet{jaggi2013}. The proofs below adapt these ingredients to exact total and per-round accounting.

Table~\ref{tab:comparison} reports static regret divided by $GD$. A strict limit applies in every physical round; a total limit applies over the horizon; an amortized count divides total work by $T$. The latter two conventions can allow calls to be concentrated at block boundaries. Let
\begin{equation}\label{eq:table-rate}
 \Psi(T,q)=\max\{\sqrt T,Tq^{-1/4}\},\qquad Q_*=\min\{Q,BT\}.
\end{equation}
The $O$ entries are upper guarantees and the $\Theta$ entries are matching minimax bounds. The dependence on dimension $d$ and smoothness $\beta$ is explicit.
\begin{table}[tbp]
\centering
\caption{Online regret with exact linear optimization access. Regret is divided by $GD$. Prior rows report upper bounds; the $\Theta$ entries give matching minimax bounds in the model of Section~\ref{sec:model}.}\label{tab:comparison}
\small\setlength{\tabcolsep}{3pt}\renewcommand{\arraystretch}{1.18}
\begin{tabular}{@{}>{\raggedright\arraybackslash}p{.185\linewidth}>{\raggedright\arraybackslash}p{.235\linewidth}>{\raggedright\arraybackslash}p{.14\linewidth}>{\raggedright\arraybackslash}p{.205\linewidth}>{\raggedright\arraybackslash}p{.175\linewidth}@{}}
\toprule
\textbf{Work} & \textbf{Regret / $GD$} & \textbf{Oracle calls} & \textbf{Function class} & \textbf{Budget} \\
\midrule
\multicolumn{5}{@{}l}{\textit{Previous upper bounds}}\\[2pt]
\citet{hazan2012} & $O(T^{3/4})$ & One per round & Convex, $G$-Lipschitz & Strict per-round \\
\addlinespace[5pt]
\citet{weibel2026} & $O(T^{3/4})$ & One per round & Convex, $G$-Lipschitz & Strict per-round \\
\addlinespace[5pt]
\citet{garber2022} & $O(T^{3/4})$ & At most $T$ total & Convex, $G$-Lipschitz & Total \\
\addlinespace[5pt]
\citet{lu2025decentralized} & $O(T^{1-\theta/2})$ & $O(T^{2\theta})$ total per agent & Convex, $G$-Lipschitz specialization & Total (batched) \\
\addlinespace[5pt]
\citet{hazan2020}\newline Linear case & $O(\sqrt{dT})$ & One per round & Linear, $G$-Lipschitz & Strict per-round \\
\addlinespace[5pt]
\citet{hazan2020}\newline Smooth case & $O\big((\sqrt d+\beta D/G)$\newline ${}\hspace{1em}T^{2/3}\big)$ & At most $T$ total & Convex, $G$-Lipschitz, $\beta$-smooth & Amortized \\
\midrule
\multicolumn{5}{@{}l}{\textit{This paper}}\\[2pt]
Theorem~\ref{thm:joint} & $\Theta\big(\Psi(T,Q_*+1)\big)$ & $Q$ total; $B$ per round & Convex, $G$-Lipschitz & Joint \\
\addlinespace[5pt]
Corollary~\ref{cor:total} & $\Theta\big(\Psi(T,Q+1)\big)$ & At most $Q$ total & Convex, $G$-Lipschitz & Total \\
\addlinespace[5pt]
Corollary~\ref{cor:strict} & $\Theta\!\big(\max\{\sqrt T,$\newline ${}\hspace{1em}T^{3/4}B^{-1/4}\}\big)$ & At most $B$ per round & Convex, $G$-Lipschitz & Strict per-round \\
\addlinespace[5pt]
Theorem~\ref{thm:smooth} & Lower: $\Omega(\Psi_\lambda(T,Q_*+1))$\newline Upper: $O(\Psi(T,Q_*+1))$ & $Q$ total; $B$ per round & Convex, $G$-Lipschitz, $\beta$-smooth & Joint \\
\bottomrule
\end{tabular}
\par\smallskip
\begin{minipage}{\linewidth}\small
Here $Q_*=\min\{Q,BT\}$, $\Psi(T,q)=\max\{\sqrt T,Tq^{-1/4}\}$, $\lambda=\beta D/(\sqrt3G)$, and $\Psi_\lambda(T,q)=\max\{\sqrt T,T\min\{q^{-1/4},(\lambda/q)^{1/3}\}\}$. Also, $0\le\theta\le1$ and $d$ is the dimension. The Lu et al. row reports convex regret for a fixed communication network. The two Hazan--Minasyan guarantees are in expectation. The smooth bounds match when $\lambda\ge(Q_*+1)^{1/4}$ or $Q_*+1\ge T^2$.
\end{minipage}
\end{table}

\section{System model and preliminaries}\label{sec:model}
\subsection{Online protocol and oracle access}\label{sec:protocol}
The learner is given a finite-dimensional Euclidean space $E$, a horizon $T\ge1$, constants $G,D>0$, oracle budgets, and a ball $C=B_E(c,r)$ with $r>0$. The unknown compact convex set $K$ satisfies $C\subseteq K$ and $\diam(K)\le D$. All further constraint information is obtained from a fixed measurable exact oracle
\begin{equation}\label{eq:oracle}
 \LOO_K(v)\in\argmin_{x\in K}\ip v x.
\end{equation}
The oracle's tie-breaking rule is fixed. Query directions are arbitrary. In round $t$, the learner chooses $x_t$ and then observes a globally defined convex $G$-Lipschitz loss $f_t:E\to\R$. Full loss feedback is allowed; the upper algorithms use only a subgradient at the played action. Static regret is
\begin{equation}\label{eq:regret}
 \Reg_T=\sum_{t=1}^T f_t(x_t)-\min_{u\in K}\sum_{t=1}^T f_t(u).
\end{equation}
All decision rules are measurable and finish each round. A randomized learner samples a seed at initialization and uses deterministic rules conditional on that seed.

\begin{definition}[Admissibility]\label{def:admissible}
A learner is admissible if it is feasible almost surely on every compact convex set, fixed exact oracle, and loss sequence consistent with the supplied data. Its total number of calls is at most the integer $Q\ge0$ almost surely. For a finite strict cap $B\ge1$, at most $B$ calls are made in each physical round. Every call is charged, with preprocessing assigned to round one. We write $B=\infty$ when only the total budget is imposed.
\end{definition}
Calls can occur before or after a round's feedback, subject to these caps and the temporal order of the observations. Calls made before choosing $x_t$ may depend only on feedback from earlier rounds. After $f_t$ is observed, further calls may use that feedback but can affect only future actions; all calls made in either part of round $t$ count toward that round's cap. The diameter is an upper bound, so a smaller body containing $C$ remains a consistent instance. Loss formulas are globally defined and have no promised relationship to the feasible set; the same formulas are used when comparing consistent bodies. The learner may perform arbitrary intermediate computation, including optimization over the hull of known feasible points. No membership, separation, or projection oracle for $K$ is supplied.

\subsection{Minimax regret and elementary inequalities}\label{sec:preliminaries}
Let $\mathcal R_{\mathrm{joint}}(T,Q,B;G,D)$ be the infimum, over admissible learners, of worst-case expected regret. The supremum ranges over finite dimensions, supplied balls, consistent bodies and fixed oracles, and deterministic oblivious loss sequences; expectation is over the learner's seed. The total-only and strict-only values are
\begin{equation}\label{eq:model-specializations}
 \begin{aligned}
 \mathcal R(T,Q;G,D)&=\mathcal R_{\mathrm{joint}}(T,Q,\infty;G,D),\\
 \mathcal R_{\mathrm{round}}(T,B;G,D)&=\mathcal R_{\mathrm{joint}}(T,BT,B;G,D).
 \end{aligned}
\end{equation}
For a prescribed $\beta\ge0$, the smooth value $\mathcal R_{\mathrm{joint},\beta}$ uses the same definition with the additional requirement
\begin{equation}\label{eq:smooth-definition}
 \norm{\nabla f_t(x)-\nabla f_t(y)}\le\beta\norm{x-y}\qquad(x,y\in E).
\end{equation}
The smoothness bound is supplied to the learner.

For any played subgradient $g_t\in\partial f_t(x_t)$, convexity and Lipschitz continuity give
\begin{equation}\label{eq:convex-prelim}
 f_t(x_t)-f_t(u)\le\ip{g_t}{x_t-u},\qquad \norm{g_t}\le G.
\end{equation}
The upper proof uses this inequality together with a squared-distance potential. The lower proof uses the fact that a nonnegative vector with total mass $a$ on at most $N$ coordinates has Euclidean norm at least $a/\sqrt N$. Both statements apply independently of the ambient dimension.

\section{Sharp regret under simultaneous oracle budgets}\label{sec:joint}
The effective total allowance $Q_*=\min\{Q,BT\}$ determines the joint-budget rate. The upper bound below enforces both budgets on every execution, while the lower bound holds for arbitrary randomized learners.
\begin{theorem}[Joint-budget characterization]\label{thm:joint}
Let $T\ge1$ and $Q\ge0$ be integers. Let $B$ be a positive integer or $\infty$, and set $q_*=1+\min\{Q,BT\}$. Then
\begin{equation}\label{eq:joint-rate}
 \frac{GD}{100}\Psi(T,q_*)
 \le \mathcal R_{\mathrm{joint}}(T,Q,B;G,D)
 \le 6GD\Psi(T,q_*).
\end{equation}
The upper bound is attained by a deterministic learner and holds for every realized adaptive loss sequence. For each learner, the lower bound is witnessed by a deterministic oblivious loss sequence on a fixed body with a fixed exact oracle.
\end{theorem}

\begin{remark}[Which budget is active]\label{rem:active}
When $Q\le BT$, the total allowance determines the rate. When $Q\ge BT$, the roundwise capacity determines it. Scheduling introduces no additional asymptotic penalty beyond this effective allowance. The rate reaches $\Theta(GD\sqrt T)$ once $q_*$ is of order $T^2$; below that scale, increasing usable oracle work by a factor of $16$ improves the oracle-dependent term by a factor of $2$.
\end{remark}

\subsection{Achievability: counted updates and interleaved blocks}\label{sec:upper}
The construction combines an approximate-gradient update with a block schedule. We first bound the number of calls required by one update, then telescope the distance potential over blocks, and finally allocate those calls to physical rounds. The stopping rule is the quadratic Frank--Wolfe gap certificate \citep{jaggi2013}; the use of blocked approximate projections follows \citet{garber2022}.

\paragraph{One feasible update.}\label{sec:step}
The ideal projected-gradient step would project $y=x-\eta g$ onto $K$. Instead, we minimize the quadratic $\norm{z-y}^2/2$ approximately over $K$ using only convex combinations of feasible points. The required certificate is a comparator-uniform distance inequality, not an accurate reconstruction of the projection. Starting from $x$ keeps the initial quadratic value small enough to bound the number of failed stopping tests.
\begin{lemma}[Counted approximate gradient step]\label{lem:step}
Let $x\in K$, $\norm g\le A$, $A>0$, $\eta>0$, and $\eta A\le D$. For any integer $H\ge1$, set $y=x-\eta g$ and $\varepsilon=\eta AD/\sqrt H$. Algorithm~\ref{alg:step} returns a feasible $x^+$ in at most $H$ oracle calls, including its successful test, with
\begin{equation}\label{eq:approx-step}
 \norm{x^+-u}^2\le\norm{x-\eta g-u}^2+2\varepsilon\qquad(u\in K).
\end{equation}
\end{lemma}
\begin{algorithm2e}[htbp]
\caption{Counted feasible update}\label{alg:step}
\KwIn{$x\in K$, $g$, $A>0$, $\norm g\le A$, $0<\eta A\le D$, integer $H\ge1$}
$y\gets x-\eta g$; $z\gets x$; $\varepsilon\gets\eta AD/\sqrt H$\;
\For{$j=1,\ldots,H$}{
 $v\gets\LOO_K(z-y)$; $\Gamma\gets\ip{z-y}{z-v}$\;
 \If{$\Gamma\le\varepsilon$}{\KwRet{$z$}\;}
 $\gamma\gets\min\{\Gamma/D^2,1\}$; $z\gets(1-\gamma)z+\gamma v$\;
}
\end{algorithm2e}
\begin{proof}
Let $h(z)=\norm{z-y}^2/2$. Initially,
\[
 h(x)=\frac{\eta^2\norm g^2}{2}\le\frac{\eta^2A^2}{2}.
\]
Each iterate is a convex combination of feasible points. Since $\norm{z-v}\le D$, quadratic expansion gives
\[
 h(z+\gamma(v-z))\le h(z)-\gamma\Gamma+\frac{\gamma^2D^2}{2}.
\]
The tolerance satisfies $\varepsilon\le D^2$. If a test fails and $\Gamma\le D^2$, taking $\gamma=\Gamma/D^2$ decreases $h$ by $\Gamma^2/(2D^2)>\varepsilon^2/(2D^2)$. If $\Gamma>D^2$, taking $\gamma=1$ decreases $h$ by more than $D^2/2\ge\varepsilon^2/(2D^2)$. Every failed test therefore decreases $h$ by more than
\[
 \frac{\varepsilon^2}{2D^2}=\frac{\eta^2A^2}{2H}.
\]
There cannot be $H$ failures, since their combined decrease would exceed the initial value of the nonnegative function $h$. A successful test consequently occurs among the first $H$ calls; its oracle call has been counted.

At that test, minimality of $v$ gives $\ip{z-y}{z-u}\le\ip{z-y}{z-v}=\Gamma\le\varepsilon$ for every $u\in K$. Thus
\[
 \norm{z-u}^2=\norm{y-u}^2-\norm{z-y}^2+2\ip{z-y}{z-u}
 \le\norm{y-u}^2+2\varepsilon,
\]
which is \eqref{eq:approx-step} with $x^+=z$.
\end{proof}

\paragraph{Balanced blocks.}\label{sec:total-upper}
Partition the $T$ rounds into $M$ nonempty consecutive blocks of lengths $b_i$ differing by at most one, with the longer blocks first. Write
\begin{equation}\label{eq:block-squares}
 S=\sum_{i=1}^M b_i^2=\frac{T^2}{M}+r_0(1-r_0/M)
 \le\frac{T^2}{M}+\frac M4,
 \qquad r_0=T-M\lfloor T/M\rfloor.
\end{equation}
Play a fixed point $x_i$ throughout block $i$, starting at $c$, and let $g_i$ be the sum of its played subgradients, so $\norm{g_i}\le Gb_i$. Choose $\eta=D/(G\sqrt S)$. On each nonterminal block apply Algorithm~\ref{alg:step} with $A_i=Gb_i$ and tolerance $\varepsilon_i=\eta DGb_i/\sqrt H$. Its condition holds since $b_i\le\sqrt S$.

\begin{algorithm2e}[htbp]
\caption{Blocked online optimization with a counted update}\label{alg:blocked}
\KwIn{Horizon $T$, integers $1\le M\le T$ and $H\ge1$, $G,D>0$, center $c\in K$}
Partition the horizon into consecutive blocks $I_i$ of balanced lengths $b_i$, longer blocks first\;
$S\gets\sum_{i=1}^M b_i^2$; $\eta\gets D/(G\sqrt S)$; $x_1\gets c$\;
\For{$i=1,\ldots,M$}{
 $g_i\gets0$\;
 \ForEach{$t\in I_i$}{
  Play $x_i$, observe $f_t$, choose $s_t\in\partial f_t(x_i)$\;
  $g_i\gets g_i+s_t$\;
 }
 \If{$i<M$}{
  $x_{i+1}\gets$ Algorithm~\ref{alg:step} with $(x_i,g_i,Gb_i,\eta,H)$\;
 }
}
\end{algorithm2e}

\begin{proposition}[Exact finite-budget upper bound]\label{prop:finite-upper}
For any integers $1\le M\le T$ and $H\ge1$, this algorithm makes at most $(M-1)H$ calls and satisfies
\begin{equation}\label{eq:finite-upper}
 \Reg_T\le GD\left(\sqrt{\sum_{i=1}^M b_i^2}+\frac{T-b_M}{\sqrt H}\right).
\end{equation}
The terminal block requires no update.
\end{proposition}
\begin{proof}
For any comparator $u\in K$, convexity bounds regret by $\sum_i\ip{g_i}{x_i-u}$. For $i<M$, expansion of \eqref{eq:approx-step} yields
\[
 \ip{g_i}{x_i-u}\le
 \frac{\norm{x_i-u}^2-\norm{x_{i+1}-u}^2}{2\eta}
 +\frac\eta2\norm{g_i}^2+\frac{\varepsilon_i}{\eta}.
\]
For $i=M$, Young's inequality gives
$\ip{g_M}{x_M-u}\le\norm{x_M-u}^2/(2\eta)+\eta\norm{g_M}^2/2$.
The sum telescopes to
\[
 \frac{D^2}{2\eta}+\frac{\eta G^2S}{2}
 +\frac{GD}{\sqrt H}\sum_{i<M}b_i
 =GD\sqrt S+\frac{GD(T-b_M)}{\sqrt H}.
\]
There are exactly $M-1$ possible updates, each using at most $H$ calls. This proves both claims, including $M=1$.
\end{proof}

\paragraph{The total-budget specialization.}
Set
\begin{equation}\label{eq:total-parameters}
 M=\min\!\left\{T,\left\lfloor\frac{1+\sqrt{1+4Q}}2\right\rfloor\right\},
 \qquad H=M.
\end{equation}
Then $M(M-1)\le Q$, so Proposition~\ref{prop:finite-upper} respects the hard budget. By \eqref{eq:block-squares} and $M\le T$,
\begin{equation}\label{eq:total-constant}
 \Reg_T\le(1+\sqrt5/2)GD\frac{T}{\sqrt M}.
\end{equation}
If $M=T$, this is already less than $3GD\sqrt T$. Otherwise maximality in \eqref{eq:total-parameters} gives $Q+1\le M(M+1)\le2M^2$, so the right side is at most
$(1+\sqrt5/2)2^{1/4}GD T/(Q+1)^{1/4}<3GD T/(Q+1)^{1/4}$.
At $Q=0$, the exact bound \eqref{eq:finite-upper} is $GDT$, with no calls. All inequalities are pathwise and comparator-uniform.
The two terms in \eqref{eq:finite-upper} explain the design. More blocks reduce the online-gradient term, approximately $GDT/\sqrt M$ for balanced blocks. More calls per update reduce the accumulated approximation term, approximately $GDT/\sqrt H$. With total work on the order of $MH$, balancing $M$ and $H$ produces the fourth-root tradeoff. The exact choice uses $M(M-1)\le Q$, because the terminal update is absent. For example, two blocks are available with two calls. The same balance underlies the prior total-budget tradeoffs in Section~\ref{sec:related}.

\paragraph{Scheduling under a strict cap.}\label{sec:schedule}
A single blocked copy may concentrate its calls at block boundaries. To distribute this work over rounds, run one copy on the odd blocks and one on the even blocks, with both copies initialized at $c$. An update based on block $i$ is computed during block $i+1$ and used in block $i+2$. Each copy therefore receives its next action before its next assigned block starts.

\begin{proposition}[Finite-call interleaved schedule]\label{prop:schedule}
Partition the horizon into $M$ balanced blocks, and choose an integer $H\ge1$. For $\ell\in\{1,2\}$, let $S_\ell$ be the sum of squared block lengths assigned to copy $\ell$. Let $W$ be the sum of the lengths of all blocks except the terminal block of each nonempty copy. If
\begin{equation}\label{eq:capacity}
 H\le B\lfloor T/M\rfloor
\end{equation}
whenever an update is required, the two-copy algorithm makes at most $(M-2)_+H$ calls, at most $B$ in any round, and satisfies
\begin{equation}\label{eq:two-copy-exact}
 \Reg_T\le GD\left(\sqrt{S_1}+\sqrt{S_2}+\frac{W}{\sqrt H}\right).
\end{equation}
Here $(M-2)_+=\max\{M-2,0\}$, and empty copies contribute zero.
\end{proposition}
\begin{proof}
Use step size $\eta_\ell=D/(G\sqrt{S_\ell})$ for each nonempty copy. At the end of a nonterminal block of that copy, its aggregate gradient is known. Execute the next counted update during the intervening block, using at most $B$ calls in each physical round. Every intervening block has at least $\lfloor T/M\rfloor$ rounds, so \eqref{eq:capacity} guarantees completion. Only one update is computed in an intervening block. For example, during block two, the odd copy computes the action for block three from the gradients of block one, while the even copy plays its initial action. During block three, the even copy computes the action for block four.

The odd and even copies have $\lceil M/2\rceil$ and $\lfloor M/2\rfloor$ blocks. Each nonempty copy omits its final update, leaving $(M-2)_+$ updates, each of at most $H$ calls. Applying the distance-potential proof of Proposition~\ref{prop:finite-upper} separately to the two copies gives \eqref{eq:two-copy-exact}, for every common comparator $u\in K$.
\end{proof}

With $H=M$, the balanced-block identity \eqref{eq:block-squares}, Cauchy--Schwarz, and $W\le T$ give
\begin{equation}\label{eq:two-copy-rate}
 \Reg_T\le(1+\sqrt{5/2})GD\frac{T}{\sqrt M},
 \qquad N\le M(M-2)_+.
\end{equation}
The term $\sqrt{S_1}+\sqrt{S_2}$ accounts for running two learning processes. The additional factor is constant because $S_1+S_2$ is the squared-length sum of the original balanced partition.

\paragraph{Joint-budget parameter choice.}
Put
\begin{equation}\label{eq:joint-parameters}
 S_0=\min\{Q+1,BT\},\qquad
 M=\min\!\left\{T,\max\!\left\{1,\left\lfloor\sqrt{S_0/2}\right\rfloor\right\}\right\},
 \qquad H=M.
\end{equation}
The total budget enters as $Q+1$ because each copy omits its terminal update, whereas the physical-round capacity has exactly $BT$ call slots and receives no such additive term. If an update is needed, then $M\ge3$ and $M^2\le S_0/2$. Since $\lfloor T/M\rfloor\ge T/(2M)$, this implies \eqref{eq:capacity} for finite $B$. Moreover, $M\ge3$ implies $S_0\ge18$, hence $Q\ge17$, and therefore
$M(M-2)_+\le M^2\le(Q+1)/2\le Q$. When $M\le2$, there are no updates and both caps hold directly.

The elementary bound $\max\{1,\lfloor x\rfloor\}\ge x/2$ for $x\ge0$, followed by $S_0\ge q_*/2$, yields
\begin{equation}\label{eq:joint-block-lower}
 M\ge\frac{\min\{T,\sqrt{S_0}\}}{2\sqrt2}
 \ge\frac{\min\{T,\sqrt{q_*}\}}4.
\end{equation}
Consequently \eqref{eq:two-copy-rate} gives
\begin{equation}\label{eq:joint-upper-proof}
 \Reg_T\le 2(1+\sqrt{5/2})GD\Psi(T,q_*)<6GD\Psi(T,q_*).
\end{equation}
All updates use gradients already observed by their own copy, so the argument holds for every realized loss sequence, including adaptive sequences. The construction has also shown exactly where the two budgets enter: $Q$ limits the number of solves, and $B$ limits the time available to finish each solve.

\subsection{Lower bound: feasible hulls and recurring support cost}\label{sec:lower}
The lower bound has two components. Universal feasibility restricts what an action can do with a finite transcript. A product-body construction then makes acquiring useful new replies necessary in many phases. Throughout, the body and oracle are fixed independently of the learner's seed and of the sampled losses.

\begin{lemma}[Feasible actions on a finite-vertex oracle]\label{lem:hull}
Let $K=\conv(V)$ for a finite set $V$ with a fixed total order, and let $C\subseteq K$ be the supplied ball. Suppose the oracle returns the first minimizing point of $V$. For every admissible learner and every fixed globally defined loss sequence, almost surely
\begin{equation}\label{eq:hull}
 x_t\in\conv(C\cup S_t),
\end{equation}
where $S_t$ is the set of replies received before $x_t$. The conclusion holds simultaneously over any finite collection of loss sequences.
\end{lemma}
\begin{proof}
For each $S\subseteq V$, fix the alternative body $K_S=\conv(C\cup S)$. It contains the same ball, is full-dimensional, and has diameter at most $D$. Define its oracle as follows: return the first minimizing point of $S$ whenever its value is no larger than the minimum over $C$; otherwise return a ball minimizer. The minimum over the empty set is $+\infty$. For a nonzero direction $v$, the ball minimizer is $c-rv/\norm v$; at $v=0$ use $c$. This defines a fixed measurable exact oracle on each alternative body before any interaction.

If the original oracle answers a query with $w\in S$, then $w$ minimizes over $C\cup S$. Every tied point in $S$ also minimizes over $K$, so none precedes $w$ in the fixed order. The alternative oracle therefore gives exactly the same reply, including ties with the ball.

Fix a loss sequence. Let $F_S$ be the probability-one set of seeds on which the learner is feasible on $K_S$ with this oracle and these same globally defined losses. There are finitely many subsets, hence $F=\bigcap_{S\subseteq V}F_S$ has probability one. For a seed in $F$, take the actual reply set $S_t$ on the original body. Running the learner on $K_{S_t}$ with the same seed reproduces the transcript up to $x_t$: each previous original reply belongs to $S_t$, and the query, action, and feedback rules consequently agree by induction. The reproduced action is feasible on $K_{S_t}$, proving \eqref{eq:hull}. A further finite intersection covers a finite collection of losses.
\end{proof}
The order of quantifiers is essential: the finite family of alternative bodies is fixed first, the events on which feasibility holds are intersected second, and the realized transcript selects $S_t$ only afterward. Conversely, every point in $\conv(C\cup S_t)$ is feasible on every consistent convex body, because it is a convex combination of known feasible points.

This conclusion does not restrict computation over the known hull. The learner may solve any auxiliary optimization over it, discard or reweight replies, and choose any adaptive query direction. What it cannot do is certify feasibility outside the hull from this transcript alone. The same restriction applies when the hard body is a product: only the oracle replies and supplied ball are available as constraint information.

\begin{proposition}[Fixed-body high-probability lower bound]\label{prop:hp}
Fix $T,Q$ and $0<\delta<1/2$, and write $q=Q+1$. There exist a fixed full-dimensional body, supplied ball, fixed deterministic exact oracle, and finite distribution of oblivious convex $G$-Lipschitz loss sequences such that every admissible learner with budget $Q$ satisfies
\begin{equation}\label{eq:hp}
 \Prob_{f,\omega}\!\left\{\Reg_T\ge\frac{GD}{40}
 \max\!\left\{\sqrt T,\frac{T}{q^{1/4}}\right\}\right\}\ge1-\delta.
\end{equation}
Every sequence has a common minimizer. For each learner, a deterministic sequence in the support satisfies the same bound over its seed alone. For $q\le T^2$, the dimension is $O(\sqrt q\log(q/\delta))$; at $\delta=1/4$ the supplied radius is at least $D/(256\sqrt{6q})$. For $q>T^2$, dimension $O(T\log(T/\delta))$ suffices.
\end{proposition}
\begin{proof}
\paragraph{Construction and common minimizer.}\label{sec:construction}
The cube coordinates carry fresh phase information, while the simplex coordinates will penalize concentration on too few newly acquired replies. Using separate cube blocks lets one comparator optimize all phases simultaneously; using a shared simplex ensures that the cost depends on how much weight an action places on its new replies.

Assume $G=D=1$ and $q=Q+1\le T^2$. Set
\begin{equation}\label{eq:parameters}
 m=\lceil\sqrt q\rceil,\quad L=\lfloor T/m\rfloor,\quad n=64m,\quad
 k=\left\lceil8\log\frac{mq}{\delta}\right\rceil,\quad a=1/\sqrt6.
\end{equation}
Work temporarily in the affine Euclidean space
$\calH=\R^{mk}\times\{w\in\R^n:\one^\top w=a\}$. With
$\Delta_n=\{w\ge0:\one^\top w=1\}$, define
\begin{align}
 K&=\left[-\frac a{\sqrt{mk}},\frac a{\sqrt{mk}}\right]^{mk}\times a\Delta_n,
 &c&=(0,a\one/n),\label{eq:body}\\
 C&=B_{\calH}(c,r),
 &r&=\frac{a}{2\max\{\sqrt{mk},n\}}.\label{eq:ball}
\end{align}
The cube and simplex have diameters $2a$ and $\sqrt2a$, respectively, so $K$ has diameter one in the product metric. A radius-$r$ displacement from $c$ changes every coordinate by at most $r$: the cube constraints remain satisfied and every simplex coordinate stays at least $a/(2n)$. Thus $C\subseteq K$, and $K$ is full-dimensional in $\calH$, of dimension $mk+n-1$. An affine isometry identifies this space with an ordinary Euclidean space without changing any oracle or distance statement. Fix an order on the vertices
$(as/\sqrt{mk},ae_\ell)$, $s\in\{-1,1\}^{mk}$, and always return the first minimizer.

Sample independent uniform sign vectors $\sigma_j\in\{-1,1\}^k$ and set $u_j=\sigma_j/\sqrt k$. In phase $j$, consisting of $L$ rounds, reveal the loss
\begin{equation}\label{eq:loss}
 f_t(z,w)=\frac1{\sqrt2}\bigl(\ip{u_j}{z_j}+\norm w\bigr).
\end{equation}
The remaining $T-mL$ losses are zero. These are globally convex and $1$-Lipschitz on $\calH$, since the two components in \eqref{eq:loss} act on orthogonal coordinate spaces and each has Lipschitz constant one. The point
\begin{equation}\label{eq:comparator}
 z_j^*=-a\sigma_j/\sqrt{mk}\quad(j=1,\ldots,m),\qquad w^*=a\one/n
\end{equation}
minimizes every loss over $K$: the optimal cube value is $-a/\sqrt m$, and the minimum simplex norm is $a/\sqrt n$. All instantaneous regret terms are consequently nonnegative.

\paragraph{Previously acquired replies.}
For independent uniform signs $\xi_i$ and any fixed $s\in[-1,1]^k$, the exponential-moment calculation
\[
 \E\exp\!\left(-\tfrac12\sum_i s_i\xi_i\right)
 =\prod_i\cosh(s_i/2)\le e^{k/8}
\]
gives $\Prob\{\sum_i s_i\xi_i<-k/2\}\le e^{-k/8}$ by Markov's inequality. This is the usual Hoeffding argument \citep{hoeffding1963}, and also holds conditionally when $s$ is measurable with respect to the past and the signs are independent of that past.

Fix a call-allocation convention before analyzing the transcript. Points returned during preprocessing are available before every phase. Calls after one phase's final feedback and before the next action are assigned to the next phase, as are all subsequent calls within that phase. Calls after the nonzero-loss horizon may be ignored, since their replies cannot affect its actions. At the start of phase $j$, before its first loss has been revealed, condition on the seed, past signs, and the oracle points acquired before the phase. Every such point is fixed and independent of $\sigma_j$. For such a point $v$, its cube block has the form $as_j/\sqrt{mk}$, whence
\[
 \ip{u_j}{z_j(v)}=\frac{a}{k\sqrt m}\sum_{\ell=1}^k\sigma_{j,\ell}s_{j,\ell}.
\]
Conditionally on the pre-phase history, a union bound over its at most $Q$ oracle points gives failure probability at most $Qe^{-k/8}$. Averaging this conditional probability over the history and then union bounding over the $m$ phases gives an event $\calE$ with
\begin{equation}\label{eq:concentration}
 \Prob(\calE)\ge1-mQe^{-k/8}\ge1-\delta,
 \qquad \ip{u_j}{z_j(v)}\ge-\frac{a}{2\sqrt m}
\end{equation}
for every oracle point acquired before its phase. No union bound over possible adaptive query directions is needed. Every point of $C$ satisfies the same inequality, because $\norm{z_j}\le r\le a/(2\sqrt m)$. Points acquired during the phase are not subject to this concentration argument.

\paragraph{Newly acquired replies.}
Let $N_j$ be the number of calls in phase $j$; then $\sum_jN_j\le Q$. On the probability-one event of Lemma~\ref{lem:hull}, every action in phase $j$ can be written
$x=(1-p)y_{\mathrm{pre}}+p y_{\mathrm{new}}$, $0\le p\le1$, where $y_{\mathrm{pre}}$ combines the supplied ball and oracle points acquired before the phase, and $y_{\mathrm{new}}$ combines at most $N_j$ points acquired during the phase. Allowing even points obtained later in that phase only enlarges this set of representations. Empty groups have weight zero. On $\calE$,
\begin{equation}\label{eq:signal}
 \ip{u_j}{z_j(x)-z_j^*}\ge\frac{a(1-p)}{2\sqrt m}.
\end{equation}
The simplex component of $y_{\mathrm{new}}$ has mass $a$ on at most $N_j$ coordinates. By Cauchy--Schwarz, its norm is at least $a/\sqrt{N_j}$. Nonnegative coordinates give $\norm{w(x)}\ge ap/\sqrt{N_j}$. For $N_j\ge1$,
\begin{equation}\label{eq:phase-inequality}
 f_t(x)-f_t(x^*)\ge\frac1{\sqrt2}
 \left[\frac{a(1-p)}{2\sqrt m}+\frac{ap}{\sqrt{N_j}}-\frac a{\sqrt n}\right].
\end{equation}
For $N_j=0$, set $p=0$ and omit the fraction containing $N_j$. This inequality allows arbitrary reweighting of the full oracle history. Its two positive terms cannot be avoided simultaneously: reducing the pre-phase weight $1-p$ requires increasing the newly acquired weight $p$. If the phase has only $N_j$ newly acquired points, that increase costs $p/\sqrt{N_j}$ in the shared simplex. Equivalently, minimizing the positive part over $p\in[0,1]$ selects the smaller of the two coefficients; it does not remove both. For each realized transcript, enlarging the action hull to include all points eventually obtained in the current phase can only strengthen the learner, so the inequality also holds for the original causal action.

\paragraph{Summing the recurring cost.}
If $N_j\le4m$, the first two terms in brackets in \eqref{eq:phase-inequality} sum to at least $a/(2\sqrt m)$, whereas $a/\sqrt n=a/(8\sqrt m)$. Every round in that phase has regret at least $3/(16\sqrt{3m})$. Since $Q<m^2$, fewer than $m/4$ phases have $N_j>4m$. The other phases contribute nonnegatively, and $m\le T$, $mL\ge T/2$, and $m\le2\sqrt q$. Therefore, on $\calE$,
\begin{equation}\label{eq:lower-sum}
 \Reg_T\ge\frac{3m}{4}L\frac{3}{16\sqrt{3m}}
 \ge\frac{9}{128\sqrt6}\frac{T}{q^{1/4}}
 >\frac1{40}\frac{T}{q^{1/4}}.
\end{equation}
The parameter choice now has a direct interpretation. With $m$ phases, buying more than $4m$ replies in most phases would require order $m^2$ calls, which the budget does not provide. A constant fraction of phases therefore retain per-round loss of order $m^{-1/2}$. Choosing $m$ on the order of $\sqrt{Q+1}$ yields the fourth-root dependence.

This includes $Q=0$ and $T=1$. At zero budget the concentration event is automatic and all actions lie in the supplied ball.

\paragraph{Fixed losses, geometry, and the statistical term.}\label{sec:completion}
The body, supplied ball, and oracle above were fixed before either the signs or the seed. The sign distribution has finite support. For each learner, averaging the probability of \eqref{eq:lower-sum} over that support gives some fixed sign sequence for which the success probability over the seed alone is at least $1-\delta$. This sequence may depend on the learner, but not on its realized seed. No minimax interchange or seed-dependent choice of the body is used.

The dimension is $mk+64m-1=O(\sqrt q\log(q/\delta))$. At $\delta=1/4$, $q\le m^2$ implies
\[
 k\le8\log(4m^3)+1\le13+24\log m\le24m.
\]
Thus $\sqrt{mk}\le n$, and $r=a/(128m)\ge1/(256\sqrt{6q})$. Scaling the centered Euclidean coordinates by $D$ and using losses $\widetilde f_t(x)=GD f_t(x/D)$ gives the stated diameter, Lipschitz constant, radius, and regret.

A second fixed body supplies the statistical term when $q>T^2$. Set $k=\lceil8\log(T/\delta)\rceil$, $b=D/(2\sqrt{Tk})$, and $K=[-b,b]^{Tk}$. In round $t$, reveal
\begin{equation}\label{eq:statistical}
 f_t(x)=\frac G{\sqrt k}\ip{\sigma_t}{x^{(t)}},
\end{equation}
where the $T$ blocks are disjoint and the sign vectors are independent. The common minimizer has blocks $u^{(t)}=-b\sigma_t$ and loss $-GD/(2\sqrt T)$ in every round. Before feedback, feasibility gives $x_t^{(t)}=b s_t$ with $s_t\in[-1,1]^k$ independent of the fresh signs. The same conditional sign bound shows that, with probability at least $1-Te^{-k/8}\ge1-\delta$, every round has loss at least $-GD/(4\sqrt T)$. Hence $\Reg_T\ge GD\sqrt T/4$. This argument holds even with an explicit description of the cube and unlimited calls. Its dimension is $Tk$; supply the centered radius-$b$ ball and fix an ordered vertex oracle.

For $q>T^2$, the maximum in \eqref{eq:hp} equals $\sqrt T$, completing the bound in this branch. The finite-support averaging argument extracts a deterministic oblivious loss sequence for each learner, and all its per-round regret terms remain nonnegative.
\end{proof}

\begin{proof}[Proof of Theorem~\ref{thm:joint}]
The construction in \eqref{eq:joint-parameters}--\eqref{eq:joint-upper-proof} proves the upper bound and respects both caps. For the lower bound, every jointly constrained learner makes at most $Q_*$ calls in total. Apply Proposition~\ref{prop:hp} with this total allowance and $\delta=1/4$. The fixed sequence extracted for that learner has nonnegative regret on every feasible execution and satisfies the high-probability lower bound. Hence
\[
 \E\Reg_T\ge\frac{3GD}{160}\Psi(T,Q_*+1)
 \ge\frac{GD}{100}\Psi(T,Q_*+1).
\]
Taking the worst case over instances and then the infimum over learners proves \eqref{eq:joint-rate}.
\end{proof}

\subsection{Consequences and recovery of previous rates}\label{sec:consequences}
The total-only and strict-only problems are special cases of Theorem~\ref{thm:joint}. Their upper constants can be sharpened using the constructions already analyzed.

\begin{corollary}[Total oracle budget]\label{cor:total}
For every integer $Q\ge0$,
\begin{equation}\label{eq:minimax}
 \frac{GD}{100}\Psi(T,Q+1)\le\mathcal R(T,Q;G,D)
 \le3GD\Psi(T,Q+1).
\end{equation}
\end{corollary}
\begin{proof}
Set $B=\infty$ in Theorem~\ref{thm:joint} for the lower bound. The single-copy choice \eqref{eq:total-parameters} and bound \eqref{eq:total-constant} give the upper constant $3$.
\end{proof}

\begin{corollary}[Strict per-round budget]\label{cor:strict}
For every integer $B\ge1$,
\begin{equation}\label{eq:strict-rate}
 \begin{aligned}
 \frac{GD}{100\,2^{1/4}}\max\!\left\{\sqrt T,\frac{T^{3/4}}{B^{1/4}}\right\}
 &\le\mathcal R_{\mathrm{round}}(T,B;G,D)\\
 &\le5GD\max\!\left\{\sqrt T,\frac{T^{3/4}}{B^{1/4}}\right\}.
 \end{aligned}
\end{equation}
\end{corollary}
\begin{proof}
Set $Q=BT$ in Theorem~\ref{thm:joint}. Since $BT+1\le2BT$, its lower bound implies the stated lower constant. The schedule \eqref{eq:joint-parameters} now has $S_0=BT$, so
$M\ge\min\{T,\sqrt{BT}\}/(2\sqrt2)$. In \eqref{eq:two-copy-rate}, this gives upper constant $(1+\sqrt{5/2})\sqrt{2\sqrt2}<5$.
\end{proof}

\begin{remark}[The one-call operating point]\label{rem:one-call}
At $B=1$, Corollary~\ref{cor:strict} gives $\Theta(GDT^{3/4})$. Its upper order agrees with online Frank--Wolfe \citep{hazan2012,weibel2026} and with the linear-total-budget guarantees \citep{garber2022,lu2025decentralized}. The lower bound establishes that this order is minimax over all admissible randomized learners. More generally, increasing the strict cap from $1$ to $B$ improves the oracle-dependent rate by $B^{1/4}$ until the $\sqrt T$ term dominates.
\end{remark}

\begin{corollary}[Budgets for a target regret]\label{cor:target}
Let $1/2\le\alpha\le1$ and $q_*=1+\min\{Q,BT\}$. A uniform guarantee of order $GDT^\alpha$ is possible if and only if
\begin{equation}\label{eq:target}
 q_*=\Omega\bigl(T^{4(1-\alpha)}\bigr).
\end{equation}
More explicitly, a guarantee $\E\Reg_T\le C_0GDT^\alpha$ requires
\begin{equation}\label{eq:target-necessary}
 q_*\ge \frac{T^{4(1-\alpha)}}{(100C_0)^4},
\end{equation}
and $q_*\ge T^{4(1-\alpha)}$ suffices for regret at most $6GDT^\alpha$. With only a total cap, the minimum sufficient order of $Q+1$ is $T^{4(1-\alpha)}$. With only a strict cap, the minimum sufficient order of $B$ is $\max\{1,T^{3-4\alpha}\}$.
\end{corollary}
\begin{proof}
The lower half of Theorem~\ref{thm:joint} implies
$C_0T^\alpha\ge T/(100q_*^{1/4})$, which rearranges to \eqref{eq:target-necessary}. Conversely, $q_*\ge T^{4(1-\alpha)}$ bounds the oracle-dependent term by $T^\alpha$, and $\sqrt T\le T^\alpha$. This proves the sufficient bound. The total-only assertion sets $B=\infty$. For a strict-only budget, $q_*=BT+1$; inversion gives $B=\Omega(T^{3-4\alpha})$ when that power grows, and the smallest allowed cap is $1$. The matching choices follow from Corollary~\ref{cor:strict}.
\end{proof}
At $\alpha=1$, zero calls suffice by playing $c$ throughout. At $\alpha=3/4$, a linear total budget and one call per round suffice. At $\alpha=1/2$, the necessary and sufficient scales are $Q=\Theta(T^2)$ and $B=\Theta(T)$ when both restrictions are imposed.

\begin{corollary}[The polynomial tradeoff under joint caps]\label{cor:polynomial}
For $0\le\theta\le1$, take
\begin{equation}\label{eq:theta-budgets}
 Q_\theta=\lceil T^{2\theta}\rceil-1,\qquad
 B_\theta=\max\{1,\lceil T^{2\theta-1}\rceil\}.
\end{equation}
Then
\begin{equation}\label{eq:theta-rate}
 \mathcal R_{\mathrm{joint}}(T,Q_\theta,B_\theta;G,D)
 =\Theta(GDT^{1-\theta/2}).
\end{equation}
\end{corollary}
\begin{proof}
Because $B_\theta T\ge T^{2\theta}>Q_\theta$, the effective allowance satisfies
$q_*=Q_\theta+1=\lceil T^{2\theta}\rceil$. Hence
$T^{2\theta}\le q_*\le2T^{2\theta}$, while
$T^{1-\theta/2}\ge\sqrt T$. Equation~\eqref{eq:theta-rate} follows from Theorem~\ref{thm:joint}.
\end{proof}
The pair of total-work and regret exponents in \eqref{eq:theta-budgets}--\eqref{eq:theta-rate} recovers the convex case of \citet{lu2025decentralized}. The corollary additionally imposes the stated strict cap and supplies the matching lower bound. It concerns the centralized oracle problem; communication requirements in the decentralized setting are separate.

\section{Joint budgets with prescribed smoothness}\label{sec:smooth}
A prescribed smoothness bound changes the lower-bound construction because the norm penalty becomes increasingly curved as the simplex dimension grows. An analytic replacement preserves the same phase mechanism while controlling curvature. The upper algorithm of Section~\ref{sec:upper} remains valid for this smaller loss class.

Define
\begin{equation}\label{eq:smooth-rate-definition}
 \lambda=\frac{\beta D}{\sqrt3G},\qquad
 \Psi_\lambda(T,q)=\max\!\left\{\sqrt T,
 T\min\!\left\{q^{-1/4},(\lambda/q)^{1/3}\right\}\right\}.
\end{equation}
\begin{theorem}[Smooth losses under joint budgets]\label{thm:smooth}
For every $\beta\ge0$ and the budgets of Theorem~\ref{thm:joint}, let $q_*=1+\min\{Q,BT\}$. Then
\begin{equation}\label{eq:smooth}
 \frac{GD}{240}\Psi_\lambda(T,q_*)
 \le\mathcal R_{\mathrm{joint},\beta}(T,Q,B;G,D)
 \le6GD\Psi(T,q_*).
\end{equation}
The upper algorithm is deterministic. For every $0<\delta<1/2$, a fixed body, fixed exact oracle, and finite distribution of globally real-analytic convex $G$-Lipschitz losses with gradient Lipschitz constant at most $\beta$ satisfy
\begin{equation}\label{eq:smooth-hp}
 \Prob_{f,\omega}\!\left\{\Reg_T\ge\frac{GD}{160}\Psi_\lambda(T,q_*)\right\}\ge1-\delta
\end{equation}
for every admissible jointly constrained learner. Every sequence has a common minimizer, and a deterministic oblivious sequence with this probability guarantee over the seed alone can be extracted for each learner.
\end{theorem}
The two sides coincide in order when $\lambda\ge q_*^{1/4}$. Below this threshold, the theorem gives a curvature-sensitive lower bound; the upper bound remains the general convex guarantee. The threshold and the resulting budget requirements are stated explicitly after the proof.

\subsection{Analytic construction and proof}\label{sec:smooth-proof}
On the affine ambient space in \eqref{eq:body}, $\norm w\ge a/\sqrt n$ even outside the simplex. Thus the original norm penalty is analytic there, with gradient Lipschitz constant at most $\sqrt n/a$. Including the factor $1/\sqrt2$ gives smoothness at most $\sqrt{3n}$, which grows with the oracle budget. To prescribe the curvature independently, for $\rho>0$ define
\begin{equation}\label{eq:radial}
 P_\rho(w)=\sqrt{\rho^2+\norm w^2}-\rho.
\end{equation}
\begin{lemma}[Regularity and quadratic estimates]\label{lem:radial}
The function $P_\rho$ is globally real-analytic, convex, $1$-Lipschitz, and has $1/\rho$-Lipschitz gradient. It is increasing in $\norm w$, and
\begin{equation}\label{eq:radial-estimates}
 P_\rho(w)\le\frac{\norm w^2}{2\rho},\qquad
 P_\rho(w)\ge\frac{8\norm w^2}{17\rho}\quad\text{when }\norm w\le\rho/2.
\end{equation}
\end{lemma}
\begin{proof}
Let $s=\sqrt{\rho^2+\norm w^2}$. Direct differentiation gives
\[
 \nabla P_\rho(w)=\frac w s,\qquad
 \nabla^2P_\rho(w)=\frac I s-\frac{ww^\top}{s^3}.
\]
The Hessian is positive semidefinite, with eigenvalues at most $1/s\le1/\rho$, and the gradient norm is at most one. The strictly positive quantity inside the square root also proves real analyticity at every point. Monotonicity in the radius is immediate. Finally,
\[
 P_\rho(w)=\frac{\norm w^2}{\sqrt{\rho^2+\norm w^2}+\rho}.
\]
The denominator is at least $2\rho$. When $\norm w\le\rho/2$, it is at most $(1+\sqrt5/2)\rho<17\rho/8$. These observations prove \eqref{eq:radial-estimates}.
\end{proof}
The penalty is used to prescribe the regularity of the losses in the lower-bound instance.

\paragraph{Why the smooth exponent changes.}
In normalized coordinates, write the penalty scale as $\rho=a/h$. At small radii, $P_\rho$ is comparable to $\norm w^2/\rho$. Placing weight $p$ on $N$ newly acquired simplex vertices therefore gives a penalty on the order of $hp^2/N$, whereas pre-phase weight still gives a signal loss of order $(1-p)/\sqrt m$. The relevant threshold becomes $N\asymp h\sqrt m$ per phase, or $Q_*\asymp hm^{3/2}$ across $m$ phases. Solving this balance yields $T/\sqrt m\asymp T(h/Q_*)^{1/3}$ for positive $Q_*$. The proof also covers zero allowance by choosing $h=\min\{\lambda,(Q_*+1)^{1/4}\}$ and uses the statistical construction when the required phase count exceeds $T$.

\begin{proof}[Proof of Theorem~\ref{thm:smooth}]
The upper bound follows by applying the joint-budget algorithm to the smooth loss class. For the lower bound, it suffices to construct a hard instance for total allowance $Q_*=\min\{Q,BT\}$, since every jointly constrained learner satisfies that allowance. Within this proof write $Q_0=Q_*$ and $q=Q_0+1$. Normalize $G=D=1$ and suppose first that $\lambda=\beta/\sqrt3>0$. Put
\begin{equation}\label{eq:smooth-parameters}
 h=\min\{\lambda,q^{1/4}\},\qquad
 m=\left\lceil(q/h)^{2/3}\right\rceil.
\end{equation}
There is no requirement that $h\ge1$. The inequalities needed below are
\begin{equation}\label{eq:smooth-balance}
 h>0,\quad h^4\le q,\quad m\ge\sqrt q,\quad h\le\sqrt m,
 \quad q\le hm^{3/2},\quad m\le2(q/h)^{2/3}.
\end{equation}
They follow from $h\le q^{1/4}$ and $(q/h)^{2/3}\ge1$.

\paragraph{Geometry, feedback, and the common minimizer.}
First suppose $m\le T$. Choose
\begin{equation}\label{eq:smooth-geometry}
 L=\lfloor T/m\rfloor,\quad
 k=\left\lceil8\log\frac{mq}{\delta}\right\rceil,\quad
 n=\max\{2,\lceil16h\sqrt m\rceil\},\quad
 a=1/\sqrt6,\quad \rho=a/h.
\end{equation}
Use exactly the affine space, product body, ball, and ordered vertex oracle in \eqref{eq:body}--\eqref{eq:ball}, with these $m,k,n$. The geometric verification for this body only requires $m,k\ge1$ and $n\ge2$, so the diameter is one and the supplied positive-radius ball is contained in a full-dimensional body. The maximum with $2$ in \eqref{eq:smooth-geometry} keeps this true even at arbitrarily small prescribed smoothness.

Use independent sign vectors and phase directions as before, but replace \eqref{eq:loss} by
\begin{equation}\label{eq:smooth-loss}
 f_t(z,w)=\frac1{\sqrt2}\bigl(\ip{u_j}{z_j}+P_\rho(w)\bigr)
 \quad\text{in phase }j,
\end{equation}
with trailing zero losses. Lemma~\ref{lem:radial} makes these functions convex, globally analytic, and globally $1$-Lipschitz on $\calH$. Their gradient Lipschitz constant is at most
\[
 \frac{1}{\sqrt2\rho}=\sqrt3h\le\sqrt3\lambda=\beta.
\]
Restriction to an affine space projects the gradient onto its translation space and preserves the bound. The same global formulas and bounds hold on every alternative body used in Lemma~\ref{lem:hull}.

The common minimizer remains \eqref{eq:comparator}, because $P_\rho$ is increasing in the simplex norm. Its penalty satisfies
\begin{equation}\label{eq:smooth-comparator}
 P_\rho(w^*)\le\frac{a^2}{2\rho n}=\frac{ah}{2n}
 \le\frac{a}{32\sqrt m}.
\end{equation}
All per-round regret terms are again nonnegative.

\paragraph{The support cost in one phase.}
The conditional sign argument in \eqref{eq:concentration} is independent of the radial penalty, so its event still has probability at least $1-\delta$. Intersect it with the probability-one hull event. For an action of new weight $p$ in a phase with $N_j\ge1$ calls, the signal and support inequalities give
\begin{equation}\label{eq:smooth-support}
 \ip{u_j}{z_j(x)-z_j^*}\ge\frac{a(1-p)}{2\sqrt m},
 \qquad \norm{w(x)}\ge\frac{ap}{\sqrt{N_j}}.
\end{equation}
Suppose $N_j\le4h\sqrt m$ and set $s_0=ap/\sqrt{4h\sqrt m}$. Then $\norm{w(x)}\ge s_0$ and
\[
 \frac{s_0}{\rho}=\frac{p\sqrt h}{2m^{1/4}}\le\frac12.
\]
Monotonicity and the lower estimate in Lemma~\ref{lem:radial} show
\begin{equation}\label{eq:smooth-penalty}
 P_\rho(w(x))\ge\frac{8s_0^2}{17\rho}
 =\frac{2ap^2}{17\sqrt m}.
\end{equation}
Only the guaranteed lower radius $s_0$, not the learner's actual radius, is required to be in the quadratic regime. Combining \eqref{eq:smooth-support}--\eqref{eq:smooth-penalty} and subtracting \eqref{eq:smooth-comparator} gives
\begin{align}
 f_t(x)-f_t(x^*)&\ge\frac{a}{\sqrt{2m}}
 \left(\frac{1-p}{2}+\frac{2p^2}{17}-\frac1{32}\right)
 \ge\frac{47}{1088\sqrt{3m}}.\label{eq:smooth-phase}
\end{align}
The bracket decreases on $[0,1]$, and its value at $p=1$ is $47/544$. If $N_j=0$, take $p=0$ and use $P_\rho\ge0$; the same lower bound follows. This also handles a threshold $4h\sqrt m<1$, when every phase satisfying the threshold has no calls.

\paragraph{Summation and parameter balance.}
By \eqref{eq:smooth-balance}, $\sum_jN_j\le Q_0<q\le hm^{3/2}$. Thus fewer than $m/4$ phases have $N_j>4h\sqrt m$. At least $3m/4$ phases satisfy \eqref{eq:smooth-phase} throughout their $L$ rounds. Using $mL\ge T/2$ and the upper bound on $m$ in \eqref{eq:smooth-balance},
\begin{align}
 \Reg_T&\ge\frac{3m}{4}L\frac{47}{1088\sqrt{3m}}
 \ge\frac{141}{8704\sqrt6}T(h/q)^{1/3}
 >\frac1{160}T(h/q)^{1/3}.\label{eq:smooth-sum}
\end{align}
Moreover,
\[
 T(h/q)^{1/3}=T\min\{q^{-1/4},(\lambda/q)^{1/3}\}.
\]
Since $T$ is an integer, $m\le T$ is equivalent to $(q/h)^{2/3}\le T$, so this quantity is at least $\sqrt T$ in the present branch. This proves the desired high-probability maximum.

\paragraph{The statistical branch and zero smoothness.}
If $m>T$, the oracle-dependent term is smaller than $\sqrt T$. The fixed-cube linear construction \eqref{eq:statistical} gives the stronger high-probability threshold $\sqrt T/4$, with a common minimizer and gradient Lipschitz constant zero. It also covers $\lambda=0$ directly, where the oracle-dependent term in \eqref{eq:smooth} vanishes.

\paragraph{Fixed instances, dimension, and rescaling.}
In the phase branch, $n\le16m+1$, so the dimension is
$mk+n-1=O(m\log(mq/\delta))$, where
$m=\lceil\max\{\sqrt q,(q/\lambda)^{2/3}\}\rceil$.
In the statistical branch it is $O(T\log(T/\delta))$. In both cases the body and ordered oracle are fixed independently of the seed and the signs. The finite-support averaging argument in Proposition~\ref{prop:hp} extracts a deterministic oblivious loss sequence for each learner.

Scale the coordinates by $D$ and the losses by $GD$. Lipschitz constants scale by $G$ and gradient Lipschitz constants by $G/D$, so the normalized parameter is exactly $\lambda=\beta D/(\sqrt3G)$. The high-probability constant is $1/160$. At $\delta=1/4$, nonnegative regret gives expected constant $3/640\ge1/240$, proving the lower half of \eqref{eq:smooth}. The same construction gives \eqref{eq:smooth-hp}.
\end{proof}

\subsection{Smoothness threshold and oracle requirements}\label{sec:smooth-budget}
\begin{corollary}[Smoothness threshold]\label{cor:smooth-threshold}
Let $q_*=1+\min\{Q,BT\}$ and define
\begin{equation}\label{eq:beta-star}
 \beta_*=\frac{\sqrt3G}{D}\,q_*^{1/4}.
\end{equation}
If $\beta\ge\beta_*$, then
\begin{equation}\label{eq:smooth-matched}
 \mathcal R_{\mathrm{joint},\beta}(T,Q,B;G,D)
 =\Theta\bigl(GD\Psi(T,q_*)\bigr).
\end{equation}
When $q_*\ge T^2$, the same matching statement holds for every $\beta\ge0$, with rate $\Theta(GD\sqrt T)$. If $0\le\beta<\beta_*$, the curvature-sensitive lower bound is
\begin{equation}\label{eq:smooth-small}
 \mathcal R_{\mathrm{joint},\beta}(T,Q,B;G,D)
 \ge\frac{GD}{240}\max\!\left\{\sqrt T,
 T\left(\frac{\beta D}{\sqrt3Gq_*}\right)^{1/3}\right\}.
\end{equation}
\end{corollary}
\begin{proof}
The two terms in the inner minimum in \eqref{eq:smooth-rate-definition} are equal exactly when $\lambda=q_*^{1/4}$, or $\beta=\beta_*$. Above that threshold the minimum equals $q_*^{-1/4}$, so Theorem~\ref{thm:smooth} has matching lower and upper orders. Below it the minimum equals $(\lambda/q_*)^{1/3}$, giving \eqref{eq:smooth-small}. If $q_*\ge T^2$, the upper rate $\Psi(T,q_*)$ is $\sqrt T$, and the statistical lower bound matches it for all $\beta$.
\end{proof}

\begin{remark}[Fixed smoothness]\label{rem:smooth-gap}
For fixed positive $\beta D/G$ and $Q_*=\Theta(T)$, Theorem~\ref{thm:smooth} gives an $\Omega(GDT^{2/3})$ lower bound and an $O(GDT^{3/4})$ upper bound. The $T^{2/3}$ upper bound of \citet{hazan2020} has additional dimension dependence, as displayed in Table~\ref{tab:comparison}. A matching upper bound uniform over dimensions in this regime would complete the curvature-sensitive characterization. At $\beta=0$, the construction yields the statistical $\sqrt T$ lower bound.
\end{remark}

\begin{corollary}[Necessary joint budgets for smooth losses]\label{cor:smooth-target}
Let $1/2\le\alpha\le1$ and $C_0>0$. A uniform guarantee
$\E\Reg_T\le C_0GDT^\alpha$ for the $\beta$-smooth class requires
\begin{equation}\label{eq:smooth-budget}
 1+\min\{Q,BT\}\ge\min\!\left\{
 \frac{T^{4(1-\alpha)}}{(240C_0)^4},
 \frac{\lambda T^{3(1-\alpha)}}{(240C_0)^3}\right\}.
\end{equation}
In particular, for fixed positive $\lambda$, a dimension-free $O(GD\sqrt T)$ guarantee requires $Q=\Omega(T^{3/2})$ and, under a finite strict cap, $B=\Omega(T^{1/2})$. If $\lambda=\Omega(\sqrt T)$, these necessary scales become $Q=\Omega(T^2)$ and $B=\Omega(T)$, which are sufficient by Theorem~\ref{thm:joint}.
\end{corollary}
\begin{proof}
Theorem~\ref{thm:smooth} implies
\[
 \min\{q_*^{-1/4},(\lambda/q_*)^{1/3}\}
 \le240C_0T^{\alpha-1}.
\]
At least one of the two quantities on the left is at most the right side. Rearranging the corresponding inequality gives one of the two thresholds in \eqref{eq:smooth-budget}, and hence their minimum. Because $q_*=1+\min\{Q,BT\}$, each resource must exceed this necessary effective allowance up to the additive one. Setting $\alpha=1/2$ proves the stated consequences. The general convex upper bound gives sufficiency in the last regime.
\end{proof}

\begin{remark}[Total-only and strict-only smooth bounds]\label{rem:smooth-special}
Setting $B=\infty$ in Theorem~\ref{thm:smooth} gives its total-budget counterpart with $q_*=Q+1$. Setting $Q=BT$ gives its strict counterpart with $q_*=BT+1$. Thus the same threshold and necessary-budget calculation apply to both settings.
\end{remark}

\section{Conclusion}\label{sec:conclusion}
The joint-budget characterization identifies the minimax oracle cost of dimension-free online convex optimization. Universal feasibility and recurring support costs give the lower bound, while counted updates and interleaved blocks attain it under simultaneous total and per-round caps. The smooth construction quantifies how prescribed curvature changes the lower bound and identifies the regime in which the general rate remains sharp.

\label{main-end}
\bibliography{references}
\end{document}